\documentclass[runningheads]{llncs}
\usepackage[T1]{fontenc}
\usepackage{csquotes}
\usepackage{xparse}
\usepackage{amsmath}
\usepackage{amssymb}
\usepackage{graphicx}
\usepackage[ruled,vlined,linesnumbered]{algorithm2e}
\DontPrintSemicolon{}
\SetAlCapSkip{.75em}
\SetKwRepeat{Do}{do}{while}
\usepackage{todonotes}
\setuptodonotes{color=green!40}
\newcommand{\Reals}{\mathbb{R}}

\newcommand{\Nats}{\mathbb{N}}

\renewcommand{\vec}[1]{\mathbf{#1}}
\newcommand{\varvec}[1]{\boldsymbol{#1}}

\NewDocumentCommand{\NN}{o d()}{%
  \mathsf{net}_{\IfNoValueTF{#1}{\varvec{\theta}}{#1}}\IfNoValueF{#2}{{\left(#2\right)}}%
}%
\NewDocumentCommand{\policy}{o d()}{%
  \pi_{\IfNoValueTF{#1}{\varvec{\theta}}{#1}}\IfNoValueF{#2}{{\left(#2\right)}}%
}%
\NewDocumentCommand{\valC}{o d()}{%
  V_{\mathcal{C}}^{\pi_{\IfNoValueTF{#1}{\varvec{\theta}}{#1}}}\IfNoValueF{#2}{{\left(#2\right)}}%
}%
\NewDocumentCommand{\safetycritic}{o d()}{%
  \widetilde{V}_{\mathcal{C}}^{\pi_{\IfNoValueTF{#1}{\varvec{\theta}}{#1}}}\IfNoValueF{#2}{{\left(#2\right)}}%
}%

\NewDocumentCommand{\prob}{o d()}{%
  \mathbb{P}_{\IfNoValueF{#1}{#1}}\IfNoValueF{#2}{{\left[#2\right]}}
}
\usepackage{hyperref}
\usepackage{color}

\usepackage[noabbrev,capitalise]{cleveref}
\begin{document}
\title{%
Counterexample-Guided Repair of Reinforcement Learning Systems Using Safety Critics
}
\titlerunning{Counterexample-Guided Repair using Safety Critics}
%
\author{David Boetius\orcidID{0000-0002-9071-1695} \and
Stefan Leue\orcidID{0000-0002-4259-624X}}
\authorrunning{D. Boetius \and S. Leue}
%
\institute{University of Konstanz, 78457 Konstanz, Germany
\email{\{david.boetius,stefan.leue\}@uni-konstanz.de}}
\maketitle              
\begin{abstract}

Na{\"i}vely trained Deep Reinforcement Learning agents may fail to satisfy
vital safety constraints.
To avoid costly retraining, we may desire to repair a previously trained
reinforcement learning agent to obviate unsafe behaviour.
We devise a counterexample-guided repair algorithm for repairing reinforcement learning systems leveraging safety critics.
The algorithm jointly repairs a reinforcement learning agent and a safety critic using gradient-based constrained optimisation.

\keywords{Reinforcement Learning \and Safety \and Repair.}
\end{abstract}
\section{Introduction}
Deep Reinforcement Learning is at the core of several recent breakthroughs in AI~\cite{%
  FawziBalogHuangEtAl2022,MankowitzMichiZhernovEtAl2023,openai22%
}.
With the increasing abilities of reinforcement learning agents, it becomes vital to effectively constrain such agents to avoid harm, particularly in safety-critical applications.
A particularly effective class of constraints are formal guarantees on the non-occurrence of undesired behaviour (safety).
Such guarantees are obtainable through formal verification~\cite{LandersDoryab2023}.

Counterexample-guided repair is a successful iterative refinement algorithm for obtaining formally verified deep neural networks in supervised learning~\cite{BoetiusLeueSutter2023}.
The repair algorithm alternates searching counterexamples and modifying the model under repair to remove counterexamples.
A central component of the algorithm is a function that quantifies the safety of a neural network output.
In reinforcement learning, the safety of an output depends on the interaction with the environment.
Therefore, quantifying the safety of an output is expensive, making it challenging to apply counterexample-guided repair to reinforcement learning.
To overcome this challenge, we propose to learn a safety critic~\cite{BharadhwajKumarRhinehartEtAl2021,HansSchneegassSchaeferEtAl2008,SrinivasanEysenbachHaEtAl2020} to quantify safety. 
Since safety critics are themselves imperfect machine learning models, we propose to repair the safety critic alongside the actual reinforcement learning agent.

In the spirit of actor-critic reinforcement learning algorithms~\cite{SuttonBarto2018}, a safety critic learns to predict the safety of a state from gathered experience~\cite{BharadhwajKumarRhinehartEtAl2021}.
The idea of safety critics is analogous to the widely used~\cite{MnihBadiaMirzaEtAl2016,SchulmanWolskiDhariwalEtAl2017} concept of value critics that learn to predict the value of a state from experience.
We can use recorded unsafe trajectories, coupled with safe trajectories, as a starting point for learning a safety critic.
Since our aim is to \emph{repair} a reinforcement learning agent, we can assume that the agent was already found to behave unsafely and, therefore, unsafe trajectories are available.

When using a safety critic in counterexample-guided repair, it is vital that the safety critic correctly recognises new counterexamples as unsafe. 
Otherwise, counterexample-guided repair can not proceed to repair the reinforcement learning agent. 
To achieve this, we propose to iteratively repair the safety critic alongside the reinforcement learning agent, such that the safety critic correctly recognises newly found counterexamples.


Our approach allows for great generality regarding the environments in which an agent operates.
Similarly, our approach does not restrict the class of safety specifications that can be repaired.
It solely requires the ability to falsify or verify the specification given an agent and an environment.
Even in the absence of a falsifier or verifier, our approach can be used to repair a reinforcement learning agent whenever unsafe behaviour is detected.

The following section reviews the literature relevant to this paper.
Following this, we formally introduce safe reinforcement learning, safety critics, and counterexample-guided repair.
The subsequent section describes our approach of counterexample-guided repair using safety critics in more detail and rigour. 
We conclude with an outlook on future work. 
This includes an experimental evaluation.

\section{Related Work}
Reinforcement learning is the process by which an agent learns to solve a task by repeatedly interacting with an environment.
The agent leans by maximising the return it receives during the interaction.
State-of-the-art algorithms for reinforcement learning include Advantage Actor-Critic (A2C)~\cite{MnihBadiaMirzaEtAl2016}, Asynchronous Advantage Actor-Critic (A3C)~\cite{MnihBadiaMirzaEtAl2016}, and Proximal Policy Optimisation (PPO)~\cite{SchulmanWolskiDhariwalEtAl2017}.
These algorithms are based on deep neural networks as reinforcement learning agents.
Due to the prevalence of deep neural networks in state-of-the-art reinforcement learning methods, this paper is primarily concerned with deep reinforcement learning. 
However, our approach is not limited to this class of models.

In safe reinforcement learning, the agent must also respect additional safety constraints.
An overview of safe reinforcement learning is presented in~\cite{GarciaFernandez2015}.
More recent developments include shielding~\cite{AlshiekhBloemEhlersEtAl2018}, safe reinforcement learning using abstract interpretation~\cite{JinTianZhiEtAl2022,SunShoukry2021}, and safe reinforcement learning via safety critics~\cite{BharadhwajKumarRhinehartEtAl2021,HansSchneegassSchaeferEtAl2008,SrinivasanEysenbachHaEtAl2020,YangSimaoTindemansEtAl2023}.
In contrast to safe reinforcement learning, the repair of reinforcement learning agents is concerned with making an existing reinforcement agent safe.
In this paper, we apply safety critics to repairing reinforcement learning agents.

Verification of reinforcement learning systems is concerned with proving that a reinforcement learning agent behaves safely. 
Recent approaches for reinforcement learning verification build upon reachability analysis~\cite{BacciGiacobbeParker2021,IvanovWeimerAlurEtAl2019,TranYangLopezEtAl2020} and model checking~\cite{AmirSchapiraKatz2021,BacciGiacobbeParker2021,EliyahuKazakKatzEtAl2021}.
A survey of the field is provided by~\cite{LandersDoryab2023}.

In the domain of supervised learning, machine learning models are verified using Satisfiability Modulo Theories (SMT) solving~\cite{Ehlers2017,GuidottiLeofanteTacchellaEtAl2019,KatzBarrettDillEtAl2017}, Mixed Integer Linear Programming (MILP)~\cite{TjengXiaoTedrake2019}, and Branch and Bound (BaB)~\cite{BunelLuTurkaslanEtAl2020,FerrariMuellerJovanovicEtAl2022,WangZhangXuEtAl2021,ZhangWangXuEtAl2022}.
Many repair algorithms for supervised machine learning models are based on counterexample-guided repair~\cite{BoetiusLeueSutter2023}.
The approaches for removing counterexamples range from augmenting the training set~\cite{PulinaTacchella2010,TanZhuGuo2021} and constrained optimisation~\cite{BauerMarquartBoetiusLeueEtAl2021,GuidottiLeofantePulinaEtAl2019} to specialised neural network architectures~\cite{DongSunWangEtAl2020,GuidottiLeofantePulinaEtAl2019}.
Non-iterative repair approaches for neural networks include~\cite{SotoudehThakur2021b,TaoNawasMitchellEtAl2023}.

\section{Preliminaries}
This section introduces Safe Reinforcement Learning, Safety Critics, Verification, Falsification, and Repair.
While our algorithm's envisioned target is deep reinforcement learning agents (neural networks), our approach is not specific to a particular model class.

\subsection{Safe Reinforcement Learning}
Following~\cite{BharadhwajKumarRhinehartEtAl2021}, we adopt a perspective on safe reinforcement learning where unsafe behaviour may occur during training but not when an agent is deployed.
Our assumption is that we can train the agent to behave safely in a simulation where unsafe behaviour is inconsequential.
A safe reinforcement learning task is formalised as a Constrained Markov Decision Process (CMDP).
We consider CMDPs with deterministic transitions.

\begin{definition}[CMDP]
  A \emph{Constrained Markov Decision Process (CMDP) with deterministic transitions} is a tuple~\((\mathcal{S}, \mathcal{A}, P, R, \mathcal{S}_0, \mathcal{C})\), where~\(\mathcal{S}\) is the \emph{state space},~\(\mathcal{A}\) is the set of \emph{actions},~\(P: \mathcal{S} \times \mathcal{A} \to \mathcal{S}\) is the \emph{transition function},~\(R: \mathcal{S} \times \mathcal{A} \to \Reals\) is the \emph{reward},~\(\mathcal{C} \subset \mathcal{S}\) is a set of \emph{safe states}, and~\(\mathcal{S}_0 \subset \mathcal{C}\) is a set of \emph{initial states}.

  For a finite time-horizon~\(T \in \Nats\),~\(\tau = s_0, a_0, s_1, a_1, \ldots, s_T\)
  is a \emph{trajectory} of a CMDP if~\(s_0 \in \mathcal{S}_0\) and~\(s_t = P(s_{t-1}, a_{t-1})\),~\(\forall t \in \{1, \ldots, T\}\).
\end{definition}

\begin{definition}[Return]
  Given a discount factor~\(\gamma \in [0, 1]\), the return~\(G\) of a trajectory~\(\tau = s_0, a_0, s_1, a_1, \ldots, s_T\) is
  \begin{equation*}
    G(\tau) = \sum_{t=0}^{T-1} \gamma^t R(s_t, a_t).
  \end{equation*}
\end{definition}

\begin{definition}[Safe Trajectories]
  For a set of safe states~\(\mathcal{C}\), a trajectory~\(\tau = s_0, a_0, s_1, a_1, \ldots, s_T\) is \emph{safe} if~\(s_0, s_1, \ldots, s_T \in \mathcal{C}\).
  We write~\(\tau \vDash \mathcal{C}\) if~\(\tau\) is safe.
\end{definition}
Assuming a uniform distribution~\(\mathcal{U}(\mathcal{S}_0)\) over the initial states, our goal is to learn a (deterministic) parametric policy~\(\pi_{\varvec{\theta}}: \mathcal{S} \to \mathcal{A}\) that maximises the expected return while maintaining safety
\begin{equation}
  \arraycolsep=4pt
  \begin{array}{cl}
    \underset{\varvec{\theta}}{\text{maximise}} & \mathbb{E}_{s_0 \sim \mathcal{U}(\mathcal{S}_0)}[G(\tau(s_0, \policy))] \\
    \text{subject to} & \tau(s_0, \policy) \vDash \mathcal{C} \quad \forall s_0 \in \mathcal{S}_0,
  \end{array}%
  \label{eqn:safe-rl}
\end{equation}
where~\(\tau(s_0, \policy) = s_0, a_0, s_1, a_1, \ldots, s_T\) is a trajectory with~\(a_t = \policy(s_t)\),~\(\forall t \in \{0, \ldots, T-1\}\).

A parametric policy may be given, for example, by a neural network~\(\NN: \Reals^n \to \mathcal{A}\) reading a numeric representation of a state~\(\vec{x}_s \in \Reals^n\),~\(n \in \Nats\),~\(s \in \mathcal{S}\) and returning an action~\(a \in \mathcal{A}\).
In this paper, we use the terms policy and agent interchangeably.

\subsection{Safety Critics}
Safety critics learn the safety value function~\(\valC: \mathcal{S} \to \Reals\)
\begin{equation}
  \valC(s) = \min_{t \in \{0, \ldots, T\}} c(s_t),\label{eqn:vc}
\end{equation}
where \(s_0 = s\),~\(s_t = P(s_{t-1}, \policy(s_{t-1}))\),~\(\forall t \in \{1, \ldots T\}\), and~\(c: \mathcal{S} \to \Reals\) is a satisfaction function~\cite{BauerMarquartBoetiusLeueEtAl2021} or robustness measure~\cite{DonzeMaler2010}
for the safe set~\(\mathcal{C}\).
\begin{definition}[Satisfaction Function]
  A function~\(c: \mathcal{S} \to \Reals\) is a \emph{satisfaction function} of a set~\(\mathcal{C} \subseteq \mathcal{S}\) if~\(\forall s \in \mathcal{S}: c(s) \geq 0 \Leftrightarrow s \in \mathcal{C}\).
\end{definition}
The concept of a safety critic is analogous to (value) critics in actor-critic reinforcement learning~\cite{SuttonBarto2018}.
Classical (value) critics learn the \emph{value} of a state~\(V^{\policy}\).
The value is the expected return when starting in a state.
Safety critics can be learned using the methods from~\cite{BharadhwajKumarRhinehartEtAl2021,SrinivasanEysenbachHaEtAl2020,YangSimaoTindemansEtAl2023}.

\subsection{Verification, Falsification, and Repair}
Given a CMDP~\(\mathcal{M} = (\mathcal{S}, \mathcal{A}, P, R, \mathcal{S}_0, \mathcal{C})\) and a policy~\(\policy\), we are interested in the question whether the policy guarantees safety for all initial states.
A \emph{counterexample} is an initial state for which following the policy leads to unsafe states.
\begin{definition}[Counterexample]
  Given a policy~\(\policy\), a \emph{counterexample} is an initial state~\(s_0 \in \mathcal{S}_0\), such that the trajectory~\(\tau = s_0, a_0, s_1, a_1, \ldots, s_T\),~\(T \in \Nats\) with~\(a_t = \policy(s_{t-1})\),~\(\forall t \in \{1, \ldots, T-1\}\) contains an unsafe state~\(s_t \notin \mathcal{C}\) for some~\(t \in \{1, \ldots, T\}\).
\end{definition}
Since counterexamples lead to unsafe states, the safety value function~\(\valC\) of a counterexample is negative.

When considering algorithms for searching counterexamples, we differentiate between \emph{falsifiers} and \emph{verifiers}. 
While falsifiers are \emph{sound} counterexample-search algorithms, verifiers are \emph{sound} and \emph{complete}.
\begin{definition}[Soundness and Completeness]
  A counterexample-search algorithm is \emph{sound} if it only produces genuine counterexamples.
  Additionally, an algorithm is \emph{complete} if it terminates and produces a counterexample for every unsafe policy.
\end{definition}
\begin{proposition}\label{prop:verify}
  A policy~\(\policy\) is safe whenever a verifier does not produce a counterexample for~\(\policy\).
\end{proposition}
\begin{proof}
  \cref{prop:verify} follows from contraposition on the completeness of verifiers.
\end{proof}

Given an unsafe policy~\(\policy\), the task of \emph{repair} is to modify the policy to be safe while maintaining high returns.
A successful repair algorithm for supervised learning is counterexample-guided repair~\cite{BauerMarquartBoetiusLeueEtAl2021,BoetiusLeueSutter2023,TanZhuGuo2021}.
The following section introduces a counterexample-guided repair algorithm for reinforcement learning.

\section{Counterexample-Guided Repair using Safety Critics}
Existing counterexample-guided repair algorithms repair supervised learning models by alternating counterexample search and counterexample removal.
\cref{algo:cgr} describes the algorithmic skeleton of counterexample-guided repair.
This skeleton is akin to all counterexample-guided repair algorithms.

\begin{algorithm}[tbhp]
  \KwIn{CMDP~\(\mathcal{M} = (\mathcal{S}, \mathcal{A}, P, R, \mathcal{S}_0, \mathcal{C})\), Policy~\(\policy\)}
  \(\mathcal{S}^c \gets \mathrm{find\ counterexamples}(\mathcal{M}, \policy)\)\;
  \Do{\(\exists s_0 \in \mathcal{S}^c: s_0 \;\mathrm{is\ counterexample}\)}{%
    \(\varvec{\theta} \gets \mathrm{remove\ counterexamples}(\mathcal{S}^c, \policy, \mathcal{M})\)\;
    \(\mathcal{S}^c \gets \mathcal{S}^c \cup \mathrm{find\ counterexamples}(\mathcal{M}, \policy)\)\;
  }
  \caption{Counterexample-Guided Repair}\label{algo:cgr}
\end{algorithm}

When using a verifier to find counterexamples, \cref{algo:cgr} is guaranteed to produce a safe policy if it terminates~\cite{BoetiusLeueSutter2023}.
However, \cref{algo:cgr} is not generally guaranteed to terminate~\cite{BoetiusLeueSutter2023}. 
Despite this, counterexample-guided repair has proven successful in repairing deep neural networks~\cite{BauerMarquartBoetiusLeueEtAl2021} and other machine learning models~\cite{TanZhuGuo2021}.

\cref{algo:cgr} has two sub-procedures we need to instantiate for obtaining an executable algorithm: finding counterexamples and removing counterexamples.
For finding counterexamples, we can use tools for verifying reinforcement learning systems~\cite{%
  AmirSchapiraKatz2021,%
  BacciGiacobbeParker2021,%
  EliyahuKazakKatzEtAl2021,%
  IvanovWeimerAlurEtAl2019,%
  TranYangLopezEtAl2020%
} (see~\cite{LandersDoryab2023} for a survey).
In the remainder of this paper, we address removing counterexamples using safety critics.

\subsection{Removing Counterexamples}
Similarly to the supervised setting~\cite{BoetiusLeueSutter2023}, removing counterexamples corresponds to solving a constrained optimisation problem
\begin{equation}
  \arraycolsep=4pt
  \begin{array}{cl}
    \underset{\varvec{\theta}}{\text{maximise}} & \mathbb{E}_{s_0 \sim \mathcal{U}(\mathcal{S}_0)}[G(\tau(s_0, \policy))] \\
    \text{subject to} & \tau(s_0, \policy) \vDash \mathcal{C} \quad \forall s_0 \in \mathcal{S}^c,
  \end{array}%
  \label{eqn:remove-cx}
\end{equation}
where~\(\tau(s_0, \policy)\) is as in \cref{eqn:safe-rl} and~\(\mathcal{S}^c\) is a finite set of counterexamples.
In the supervised setting, we can remove counterexamples by directly solving the analogue of \cref{eqn:remove-cx} using gradient-based optimisation algorithms~\cite{BauerMarquartBoetiusLeueEtAl2021}.
However, for repairing reinforcement learning policies, checking whether a set of parameters~\(\varvec{\theta}\) is feasible for \cref{eqn:remove-cx} is expensive, as it requires simulating the CMDP.\@
Additionally, the term~\(\tau(s_0, \policy) \vDash \mathcal{C}\) suffers from exploding gradients~\cite{PhilippSongCarbonell2017} due to the repeated application of~\(\policy\) for obtaining the trajectory.
These properties of \cref{eqn:remove-cx} hinder the application of gradient-based optimisation algorithms for removing counterexamples by solving \cref{eqn:remove-cx} directly.

To obtain an algorithm for removing counterexamples, we first note that \cref{eqn:remove-cx} can equivalently be reformulated using the safety value function~\(\valC\) from \cref{eqn:vc}.
Concretely, we can replace the constraint~\(\tau(s_0, \policy) \vDash \mathcal{C}\) by~\(\valC(s_0) > 0\).
Now, when approximating~\(\valC\) using a safety critic~\(\safetycritic\), we obtain
\begin{equation}
  \arraycolsep=4pt
  \begin{array}{cl}
    \underset{\varvec{\theta}}{\text{maximise}} & \mathbb{E}_{s_0 \sim \mathcal{U}(\mathcal{S}_0)}[G(\tau(s_0, \policy))] \\
    \text{subject to} & \safetycritic(s_0) \geq 0 \quad \forall s_0 \in \mathcal{S}^c,
  \end{array}%
  \label{eqn:remove-cx-vc}
\end{equation}
While \cref{eqn:remove-cx-vc} is not equivalent to \cref{eqn:remove-cx} due to using an approximation of~\(\valC\), \cref{eqn:remove-cx-vc} can be solved using techniques such as stochastic gradient descent/ascent~\cite{EbanSchainMackeyEtAl2017} or the~\(\ell_1\) penalty function method~\cite{BauerMarquartBoetiusLeueEtAl2021,NocedalWright2006}.
To ensure that solving \cref{eqn:remove-cx-vc} actually removes counterexamples, we repair the safety critic~\(\safetycritic\) alongside the policy.

\subsection{Repairing Safety Critics}
To allow us to remove counterexamples by solving \cref{eqn:remove-cx-vc}, the safety critic~\(\safetycritic\) needs to correctly recognise the counterexamples in~\(\mathcal{S}^c\) as counterexamples.
By recognising a counterexample~\(s_0\) as a counterexample, we mean that~\(\safetycritic(s_0) < 0\).
We can ensure that the safety critic recognises all counterexamples in~\(\mathcal{S}^c\) by solving
\begin{equation}
  \arraycolsep=4pt
  \begin{array}{cl}
    \underset{\varvec{\theta}}{\text{minimise}} & J(\safetycritic) \\
    \text{subject to} & \safetycritic(s_0) < 0 \quad \forall s_0 \in \mathcal{S}^c \text{ with } \valC(s_0) < 0,
  \end{array}%
  \label{eqn:repair-safety-critic}
\end{equation}
where~\(J(\safetycritic)\) is a loss function for training the safety critic~\cite{BharadhwajKumarRhinehartEtAl2021,SrinivasanEysenbachHaEtAl2020,YangSimaoTindemansEtAl2023}.
Solving \cref{eqn:repair-safety-critic} can itself be understood as removing counterexamples \emph{of the safety critic}.
As \cref{eqn:remove-cx-vc}, we can solve \cref{eqn:repair-safety-critic} using stochastic gradient descent/ascent~\cite{EbanSchainMackeyEtAl2017} or the~\(\ell_1\) penalty function method~\cite{BauerMarquartBoetiusLeueEtAl2021,NocedalWright2006}.

\subsection{Counterexample Removal Algorithm}
\begin{algorithm}[tbp]
  \KwIn{CMDP~\(\mathcal{M} = (\mathcal{S}, \mathcal{A}, P, R, \mathcal{S}_0, \mathcal{C})\), Unsafe policy~\(\policy\), Safety Critic~\(\safetycritic\), Counterexamples~\(\mathcal{S}^c\)}
  
  \While{\(\exists s_0 \in \mathcal{S}^c: s_0 \;\mathrm{is\ counterexample}\)}{%
    update~\(\safetycritic\) by solving \cref{eqn:repair-safety-critic}\;
    update~\(\policy\) by solving \cref{eqn:remove-cx-vc} using~\(\safetycritic\)\;
  }
  \caption{Counterexample Removal}\label{algo:remove-cx}
\end{algorithm}

We propose jointly modifying the safety critic and the unsafe reinforcement learning agent~\(\policy\).
\Cref{algo:remove-cx} summarises our approach.
We first update the safety critic to recognise the available counterexamples.
This corresponds to solving \cref{eqn:repair-safety-critic}.
Using the updated safety critic, we update~\(\policy\) to remove the counterexamples by solving \cref{eqn:remove-cx-vc}.

Since the safety critic may fail to recognise a counterexample as a counterexample for the updated policy, we iterate the previous two steps until all counterexamples are removed.

In principle, this procedure may fail to terminate if the safety-critic \enquote{forgets} to recognise the counterexamples of the initial policy when being updated in the second iteration of \cref{algo:remove-cx}. 
Since the policy is updated in the first iteration of \cref{algo:remove-cx}, updating~\(\safetycritic\) in the second iteration does not consider counterexamples for the initial policy.
Therefore, the policy may revert to its initial parameters in the second iteration to evade the updated safety critic. 
This leads to an infinite loop.
However, this issue can be circumvented by including previous unsafe trajectories in \cref{eqn:repair-safety-critic}, similarly to how counterexamples are retained in \cref{algo:cgr} for later iterations to counter reintroducing counterexamples.

\section{Conclusion}
We introduce a counterexample-guided repair algorithm for reinforcement learning systems.
We leverage safety critics to circumvent costly simulations during counterexample removal.
Our approach applies to a wide range of specifications and can work with any verifier and falsifier.
The central idea of our approach is to repair the policy and the safety critic jointly. 

Future work includes evaluating our algorithm experimentally and comparing it with abstract-interpretation-based safe reinforcement learning~\cite{JinTianZhiEtAl2022,SunShoukry2021}.
Since counterexample-guided repair avoids the abstraction error of abstract interpretation, we expect that counterexample-guided repair can produce less conservative, safe reinforcement learning agents.
Additionally, our ideas are not inherently limited to reinforcement learning but can be applied whenever satisfaction functions are unavailable or costly to compute.
Exploring such applications is another direction for future research.

\begin{credits}

\subsubsection{\discintname}
The authors have no competing interests to declare that are relevant to the content of this article.
\end{credits}
%
%
%
\bibliographystyle{splncs04}
\bibliography{references}
\end{document}